\documentclass[runningheads]{llncs}

\usepackage[utf8]{inputenc}
\usepackage[T1]{fontenc}

\usepackage{graphicx}
\usepackage{xcolor}
\usepackage{fca}
\usepackage{mathtools}
\usepackage[hidelinks]{hyperref}
\usepackage{cleveref}
\usepackage{enumitem}
\usepackage{algorithm}
\usepackage{listings}
\usepackage{booktabs}
\usepackage{tikz}
\usepackage{tikzscale}

\newcommand{\K}{\mathbb{K}}
\newcommand{\A}{\mathbb{A}}

\newcommand{\SC}{\mathbb{S}}
\newcommand{\N}{\mathbb{N}}

\newcommand{\BB}{\mathfrak{B}}
\DeclareMathOperator{\cg}{cg}
\newcommand{\cb}{\cellcolor{blue!20}}

\let\subset\subseteq
\let\cref\Cref

\newif\ifhideproofs

\ifhideproofs
\usepackage{environ}
\NewEnviron{hide}{}

\fi

\begin{document}
\title{Attribute Selection using Contranominal Scales}

\author{Dominik Dürrschnabel\inst{1,2}\orcidID{0000-0002-0855-4185} \and
Maren Koyda\inst{1,2}\orcidID{0000-0002-8903-6960} \and
Gerd Stumme\inst{1,2}\orcidID{0000-0002-0570-7908}}

\authorrunning{D. Dürrschnabel et al.}

\institute{%
  Knowledge \& Data Engineering Group,
  University of Kassel, Germany
  \and
  Interdisciplinary Research Center for Information System Design,
  \mbox{University of Kassel, Germany}
  \email{duerrschnabel@cs.uni-kassel.de, koyda@cs.uni-kassel.de, stumme@cs.uni-kassel.de}
}

\maketitle

\begin{abstract}
Formal Concept Analysis (FCA) allows to analyze binary data by deriving concepts and ordering them in lattices.
One of the main goals of FCA is to enable humans to comprehend the information that is encapsulated in the data; however, the large size of concept lattices is a limiting factor for the feasibility of understanding the underlying structural properties.
The size of such a lattice depends on the number of subcontexts in the corresponding formal context that are isomorphic to a contranominal scale of high dimension.
In this work, we propose the algorithm \texttt{ContraFinder} that enables the computation of all contranominal scales of a given formal context.
Leveraging this algorithm, we introduce \texttt{$\delta$-adjusting}, a  novel approach in order to decrease the number of contranominal scales in a formal context by the selection of an appropriate attribute subset.
We demonstrate that \texttt{$\delta$-adjusting} a context reduces the size of the hereby emerging sub-semilattice and that the implication set is restricted to meaningful implications.
This is evaluated with respect to its associated knowledge by means of a classification task.
Hence, our proposed technique strongly improves understandability while preserving important conceptual structures.


\end{abstract}

\section{Introduction}
\label{sec:introduction}

One of the main objectives of Formal Concept Analysis (FCA) is to present data in a comprehensible way.
For this, the data is clustered into concepts which are then ordered in a lattice structure.
Relationships between the features are represented as implications.
However, the complexity of the corresponding concept lattice can increase exponentially in the size of the input data.
Beyond that, the size of the implication set is also exponential in the worst case, even when it is restricted to a minimal base.
As humans tend to comprehend connections in smaller chunks of data, the understandability is decreased by this exponential nature even in medium sized datasets.
That is why reducing large and complex data to meaningful substructures by eliminating redundant information enhances the application of Formal Concept Analysis.
Nested line diagrams~\cite{Wille1989} and drawing algorithms~\cite{Durrschnabel2019} can improve the readability of concept lattices by optimizing their presentation.
However, neither of them compresses the size of the datasets and thus grasping relationships in large concept lattices remains hard.
Therefore, our research question is: \textbf{How can one reduce the lattice size as much as possible by reducing the data as little as possible?}
There are different ways of reducing the data. In this paper, we focus on the removal of attributes.
The size of the concept lattice is heavily influenced by the number of its Boolean suborders.
A lattice contains such an $k$-dimensional Boolean suborder if and only if the corresponding formal context contains an $k$-dimensional contranominal scale~\cite{albano2015,Koyda2021}.
Thus, to reduce the size of the concept lattice it is reasonable to eliminate those.
However, deciding on the largest contranominal scale of a formal context is an $\mathcal{NP}$-complete problem.
Therefore, choosing sensible substructures of formal contexts which can be augmented in order to reduce the number of large contranominal scales is a challenging task.

In this work, we propose the algorithm \texttt{ContraFinder} that is more efficient then prior approaches in computing all contranominal scales in real world datasets.
This enables us to present our novel approach \texttt{$\delta$-adjusting} which focuses on the selection of an appropriate attribute subset of a formal context.
To this end, we measure the influence of each attribute with respect to the number of contranominal scales.
Hereby, a sub-semilattice is computed that preserves the meet-operation.
This provides the advantage to not only maintain all implications between the selected attributes but also does not produce false implications and thus retains underlying structure.
We conduct experiments to demonstrate that the subcontexts that arise by \texttt{$\delta$-adjusting} decrease the size of the concept lattice and the implication set while preserving underlying knowledge.
We evaluate the remaining knowledge by training a classification task.
This results in a more understandable depiction of the encapsulated data for the human mind.

\ifhideproofs
\else
This rest of the paper is organized as follows. In Section 2 we recall some basic notions from FCA and graph theory followed by a short overview over previous works in Section 3.
Subsequently, we provide the algorithm \texttt{ContraFinder} to compute the set of all contranominal scales in a formal context.
In Section 5 we define \texttt{$\delta$-adjusted} subcontexts to decrease the size of concept lattices.
Section 6 evaluates and discusses the results achieved in the previous sections.
Finally, we conclude our work in Section 7 and give an outlook for future work.
\fi

\section{Foundations}
We start this section by recalling notions from FCA~\cite{fca-book}.
A \emph{formal context} is a triple $\mathbb{K}\coloneqq(G,M,I)$, consisting of an \emph{object set} $G$, an \emph{attribute set} $M$ and a binary \emph{incidence relation} $I\subseteq G\times M$.
In this work, $G$ and $M$ are assumed to be finite.
The \emph{complementary formal context} is given by $\K^C\coloneqq(G,M,(G\times M)\setminus I)$.
The maps
$\cdot'\colon\mathcal{P}(G)\to\mathcal{P}(M),~A\mapsto A'\coloneqq
\{m\in M\mid \forall g\in A\colon (g,m)\in I\}$ and $
\cdot'\colon\mathcal{P}(M)\to\mathcal{P}(G),~ B\mapsto B'\coloneqq\{g\in
G\mid \forall m\in B\colon (g,m)\in I\}$
are called \emph{derivations}.
A pair $c=(A,B)$ with $A\subseteq G$ and $B\subseteq M$ such that $A'=B$ and $B'=A$ is called a \emph{formal concept} of the context $(G,M,I)$.
The set of all formal concepts of $\K$ is denoted by $\mathfrak{B}(\mathbb{K})$.
The pair consisting of $\BB(\K)$ and the order ${\leq} \subset ({\BB(\K)\times \BB(\K)})$ with $(A_1,B_1)\leq(A_2,B_2)$ iff $A_1\subseteq A_2$ defines the \emph{concept lattice} $\underline{\mathfrak{B}}(\mathbb{K})$.
In every lattice and thus every concept lattice each subset $U$ has a unique infimum and supremum which are denoted by $\bigwedge U$ and $\bigvee U$.
The \emph{contranominal scale} of \emph{dimension} $k$ is $\N^c_k\coloneqq(\{1,2,...,k\},\{1,2,...,k\},\neq)$.
Its concept lattice is the \emph{Boolean lattices of dimension $k$}
and consists of $2^{k}$ concepts.
Let $\mathbb{K}=(G,M,I)$. We call an attribute $m$ \emph{clarifiable} if there is an attribute $n\ne m$ with $n'=m'$.
In addition we call it \emph{reducible} if there is a set $X\subseteq M$ with $m\not\subseteq X$ and $m'=X'$.
Otherwise, we call $m$ \emph{irreducible}.
$\K$ is called \emph{attribute clarified} (\emph{attribute reduced}) if it does not contain clarifiable (reducible) attributes.
The definitions for the object set are analogous.
If $\K$ is attribute clarified and object clarified (attribute reduced and object reduced), we say $\K$ is \emph{clarified} (\emph{reduced}).
This contexts are unique up to isomorphisms.
Their concept lattices are isomorphic to $\underline{\BB}(\K)$.
A \emph{subcontext} $\mathbb{S}=(H,N,J)$ of $\mathbb{K}=(G,M,I)$ is a formal context with $H\subseteq G$, $N\subseteq M$ and $J= I\cap (H \times N)$.
We denote this by $\mathbb{S}\le\K$ and use the notion $\K[H,N]\coloneqq (H,N,I\cap (H\times N))$.
If $\SC\le\K$ with $\SC\cong\N^c_k$ we call $\SC$ a \emph{contranominal scale in $\K$}.
For a (concept) lattice $(L,\leq)$ and a subset $S\subseteq L$, $(S,\leq_{S\times S})$ is called \emph{suborder} of $(L,\leq)$
A suborder $S$ of a lattice is called a \emph{sub-meet-semilattice} if $(a,b\in S \Rightarrow (a\wedge b)\in S)$ holds.
In a formal context $\K=(G,M,I)$ with $X,Y\subseteq M$ define an \emph{implication} as $X \rightarrow Y$ with
\emph{premise} $X$ and \emph{conclusion} $Y$.
An implication is \emph{valid in $\K$} if $X' \subset Y'$.
In this case, we call $X\rightarrow Y$ an \emph{implication of $\K$}.
The set of all implications of a formal context $\K$ is denoted by $Imp(\K)$.
A minimal set $\mathcal{L}(\K)\le Imp(\K)$ defines an \emph{implication base} if every implication of $\K$ follows from $\mathcal{L}(\K)$ by composition.
An implication base of minimal size is called \emph{canonical base} of $\K$ and is denoted by $\mathcal{C}(\K)$.

Now recall some notions from graph theory. A \emph{graph} is a pair $(V,E)$ with a set of \emph{vertices} $V$ and a set of \emph{edges} $E \subset \binom{V}{2}$.
Two vertices $u,v$ are called \emph{adjacent} if $\{u,v\}\in E$.
The adjacent vertices of a vertex are called its \emph{neighbors}.
In this work graphs are undirected and have no multiple edges or loops.
A graph with two sets $S$ and $T$ with $S\cup T=V$ and $S\cap T=\emptyset$ such that there is no edge with both vertices in $S$ or both vertices in $T$ is called \emph{bipartite} and denoted by $(S,T,E)$.
A \emph{matching} in a graph is a subset of the edges such that no two edges share a vertex.
It is called \emph{induced} if no two edges share vertices with some edge not in the matching.
For a formal context $(G,M,I)$ the \emph{associated bipartite graph} is the graph where $S$ and $T$ correspond to $G$ and $M$ and the set of edges to $I$.

\section{Related Work}
\label{sec:related-work}

In the field of Formal Concept Analysis numerous approaches deal with simplifying the structure of large datasets.
Large research interest was dedicated to altering the incidence relation together with the objects and attributes in order to achieve smaller contexts.
A procedure based on a random projection is introduced in~\cite{Kumar}.
Dias and Vierira~\cite{diasreducing} investigate the replacement of similar objects by a single representative.
They evaluate this strategy by measuring the appearance of false implications on the new object set.
In the attribute case a similar approach is explored by Kuitche et~al.~\cite{Kuitche2018}.
Similar to our method, many common prior approaches are based on the selection of subcontexts.
For example, Hanika et~al.~\cite{hanika2019relevant} rate attributes based on the distribution of the objects in the concepts and select a small relevant subset of them.
A different approach is to select a subset of concepts from the concept lattice.
While it is possible to sample concepts randomly \cite{sampling}, the selection of concepts by using measures is well investigated.
To this end, a structural approach is given in \cite{duffus} through dismantling where a sublattice is chosen by the iterative elimination of all doubly irreducible concepts.
Kuznetsov~\cite{kuzuetsov1990stability} proposes a stability measure for formal concepts based on the sizes of the concepts.
The support measure is used by Stumme et~al.~\cite{STUMME2002189} to generate iceberg lattices.
Our approach follows up on this, as we also preserve sub-semilattices of the original concept lattice.
However, we are not restricted to the selection of iceberg lattices.
Compared to many other approaches we do not alter the incidence or the objects and thus do not introduce false implications.

\section{Computing Contranominal Scales}
\label{sec:comp-contr}

In this section, we examine the complexity of computing all contranominals and provide the recursive backtracking algorithm \texttt{ContraFinder} to solve this task.

\subsection{Computing Contranominals is Hard}
\label{sec:problem}

The problem of computing contranominal scales is closely related to the problem of computing cliques in graphs induced maximum matchings in bipartite graphs.

The relationship between the induced matching problem and the contranominal scale problem follows directly from their respective definitions.

\begin{lemma}
\label{mimcns}
  Let $(S,T,E)$ be a bipartite graph,
  $\K\coloneqq(S,T,(S \times T)\backslash E)$ a formal context and $H\subset S, N \subset T$. The edges between $H$ and $N$ are an induced matching of size $k$ in $(S,T,E)$ iff $\K[H,N]$ is a contranominal scale of dimension~$k$.
\end{lemma}

\begin{proof}
The complement context of $(G,M,I)$ corresponds to the bipartite graph $(G,M,(G \times M)\backslash I)$.
Thus, the statement follows by the definitions of induced matching and contranominal scale. \null\hfill$\square$
\end{proof}

To investigate the connection between the clique problem and the contranominal scale problem, define the conflict graph as follows:

\begin{definition}
  Let $\K\coloneqq(G,M,I)$ be a formal context. Define the \emph{conflict graph} of $\K$ as the graph $\cg(\K)\coloneqq (V,E)$ with the vertex set $V=(G\times M)\backslash I$ and the edge set $E=\{\{(g,m),(h,n)\}\in \binom{V}{2}\mid (g,n)\in I, (h,m) \in I\}$.
\end{definition}

The relationship between the cliques in the conflict graph and the contranominal scales in the formal context is given through the following lemma.

\begin{lemma}
\label{lem:clique}
  Let $\K=(G,M,I)$ be a formal context, $cg(\K)$ its conflict graph and $H \subset G, N \subset M$.
  Then $\K[H,N]$  is a contranominal scale of dimension $k$
      iff $(H\times N) \backslash I$ is a clique of size $k$ in $\cg(\K)$.
\end{lemma}

\begin{proof}
  ``$\Rightarrow$''. Let $\K=(G,M,I)$ a formal context and $\mathbb{S}=\K[H,N]$ a contranominal scale of dimension $k$ such that $H=\{h_1,h_2,\ldots, h_k\}$, $N=\{n_1,n_2,\ldots, n_k\}$ and
  $(h_i,n_j)\in I$ iff $i\not = j$.
  As $(h_i,n_i)\not\in I$ for all $i \in \{1,2,\ldots, k\}$, the
  graph $\cg(\K)$ contains all elements
  $(h_i,n_i)$ as vertices.
  Assume two such vertices, without loss of generality $(h_1,n_1)$ and $(h_2,n_2)$, are not connected by an edge in $\cg(\K)$.
  Then either $(h_1,n_2)\not\in I$ or
  $(h_2,n_1)\not\in I$, a contradiction to $\mathbb{S}$ being contranominal.

  \noindent``$\Leftarrow$''. Let $\{(h_1,n_1),(h_2,n_2),\ldots,(h_k,n_k)\}$ be
  the vertex set of the clique of size $k$ in $\cg(\K)$.
  Then $(h_i ,n_i)\not\in I$  and $h_i,n_j \in I$ for $i \neq j$  by definition of the conflict graph.
  Thus $\SC\coloneqq(\{h_1,h_2,\ldots, h_k\},\{n_1,n_2,\ldots, n_k\},\{(h_i,n_j) \mid i\not = j\})\le \K$ and $\SC$ is a contranominal scale.
\null\hfill $\square$
\end{proof}

Furthermore, all three problems are in the same computational class as the clique problem is $NP$-complete \cite{Karp72} and Lozin~\cite{Lozin2002} shows the similar result for the induced matching problem in the bipartite case. Thus, \cref{mimcns} provides the following:

\begin{proposition}
    \label{prop:NP}
  Deciding the \texttt{CONTRANOMINAL PROBLEM} is $NP$-complete.
\end{proposition}

\subsection{Baseline Algorithms}

Building on \cref{lem:clique} the set of all contranominal scales can be computed using algorithms for iterating all cliques in the conflict graph.
The set of all cliques then corresponds to the set of all contranominal scales in the formal context.
An algorithm to iterate all cliques in a graph is proposed by Bron and Kerbosch~\cite{Bron1973}.

An alternative approach is to use branch and search algorithms such as~\cite{Xiao2017}.
Those exploit the fact that for each maximum matching and each vertex there is either an adjacent edge to this vertex in the matching or each of its neighboring vertices has an adjacent edge in the matching.
Branching on these vertices the size of the graph is iteratively decreased.
Note, that this idea, in contrast to our approach described below, does not exploit bipartiteness of the graph.

\subsection{ContraFinder: An Algorithm to Compute Contranominal Scales}
In this section we introduce the recursive backtracking algorithm \texttt{ContraFinder} to compute  all contranominal scales.
Due to \cref{prop:NP}, it has exponential runtime, thus two speedup techniques are proposed in the subsequent section.

The main idea behind \texttt{ContraFinder} is the following.
In each recursion step a set of tuples corresponding to an attribute set is investigated:

\begin{definition}
  Let $\K=(G,M,I)$ be a formal context and $N \subset M$.
  Define $C(N)\coloneqq\{(g,m)\not\in I\mid g\in G, m\in N$ and $\forall x\in N \setminus \{m\}: (g,x)\in I \}$ as the set of \emph{characterizing tuples} of $N$.
 We call $N$ the \emph{generator} of $C(N)$.
\end{definition}

The characterizing tuples encodes all contranominal scales for this attributes:

\begin{algorithm}[t]
\begin{tabular}{ll}
     \textbf{Input:}&  Formal Context $\K=(G,M,I)$  \\
     \textbf{Output:}& Set of all Contranominal Scales
\end{tabular}
\hrule
\begin{lstlisting}
def compute_contranominal_scales($G,M,I$):
  characterizing_tuples($\emptyset, M,\emptyset,  I$)

def characterizing_tuples($C_N, \smash{\tilde{M}}, F, I$):
  for $m$ in $\smash{\tilde{M}}$ in lexicographical order:
    $\smash{\tilde{M}}$ = $\smash{\tilde{M}} \setminus \{m\}$
    $cand\_C_N$ = $\{(g,n)\in C_N \mid (g,m) \in I\}$
    $cand\_m$ = $\{(g,m) \mid (g,m) \not\in I, g \not\in F, \nexists n :(g,n)\in C_N\}$
    if $|\{g \mid (g,n)\in C_N\}|=|\{g\mid (g,n)\in cand\_C_N\}|$ and $|cand\_m| > 0$:
      unpack_contranominals($cand\_C_N\cup cand\_m$)
      $C_{N_\text{new}}$ =  $cand\_C_N\cup cand\_m$
      $F_{\text{new}}$ = $F \cup \{ g\in G \mid (g,m)\not\in I\}$
      characterizing_tuples($C_{N_\text{new}}, \smash{\tilde{M}}, F_{\text{new}} , I$)

def unpack_contranominals($C_N$):
  $N$ = $\{m \mid (g,m) \in C_N\}$
  for O in $\{\{g_{m_1},\ldots, g_{m_{|N|}}\}\mid m_i \in N, g_{m_i} \in \{g \in G \mid (g,m_i)\in C_N\} \}$
    report $(O, N)$ as contranominal scale
\end{lstlisting}
  \caption{\texttt{ContraFinder}}
  \label{alg:cnc}
\end{algorithm}

\begin{lemma}
\label{lem:problem}
Let $\K=(G,M,I)$, $N\subseteq M$ and $H(m)\coloneqq\{g \in G \mid (g,m)\in C(N)\}$. Then $\K[O,N]$ is a contranominal scale iff $O$ contains exactly one element of each $H(m)$ with $m\in N$.
\end{lemma}

\begin{proof}
``$\Rightarrow$'' Let $O=\{g_1,\ldots,g_{|N|}\}$ such that it contains exactly one element of each $H(m)$.
Then, for every object $g\in O$, there is exactly one $m\in N$ with $(g,m)\not\in I$ due to the definition of $C(N)$.
Also, $|O|=|N|$ as $H(m) \cap H(n) = \emptyset$ for distinct $m,n \in M$.
Thus the context $\K[O,N]$ is a contranominal scale.

``$\Leftarrow$'' Now let $\SC = \K[O,N]$ be a contranominal scale.
By definition for all elements $(h,n)\in C(N)$ it holds $(h,n) \not \in I$.
Because $\SC$ is contranominal, there is no attribute $m\in N$ with two objects $g,h \in O$ such that $(g,m),(h,m)\not \in I$.
Thus $O$ contains exactly one element of each $H(m)$ with $m\in N$.
\null\hfill$\square$
\end{proof}

\cref{lem:problem} implies that such contranominal scales can exist only if no $H(m)$ is empty and $|N|=|O|$.
Both this sets can be reconstructed from a set of characterizing tuples corresponding to $N$.
This is done in  \texttt{unpack\_contranominals} in \cref{alg:cnc}.
Therefore, $N$ does not have to be memorized in \texttt{ContraFinder}.
The algorithm exploits the fact that for each set of characterizing tuples $C(N)$ the attributes $N$ can be ordered and iterated in lexicographical order, similar to \texttt{NextClosure}~\cite[sec. 2.1]{fca-book}.

\begin{definition}
    Let $(M,\le)$ be a linearly ordered set.
    The \emph{lexicographical order} on $\mathcal{P}(M)$ is a linear order.
    Let $A={a_1, \ldots ,a_n}$ and $B= {b_1, \ldots ,b_m}$ with $a_i < a_{i+1}$ and $b_i < b_{i+1}$.
    $A < B$ in case $n < m$ if $(a_1,\ldots ,a_n)=(b_1,\ldots ,b_n)$ and in case $n=m$ if $\exists i: \forall j \leq i: a_j = b_j \text{ and } a_i < b_i$.
\end{definition}

Similar to \texttt{Titanic},  our algorithm utilises the following anti-monotonic property.
Each contranominal scale of dimension $k$ has a contranominal scale of of dimension $k-1$ as subcontext.
Thus, only attribute combinations $N$ have to be considered if $\forall N' \subset N: C(N')\neq \emptyset$.
The algorithm removes in each recursion step the attributes in $\tilde{M}$ in lexicographical order to guarantee that all attribute combinations  of the formal context with contranominal scales are investigated.

In each step the set of forbidden objects $F$ increases, since each contranominal scale contains exactly one non-incidence in each contained object.

\begin{theorem}
  The algorithm reports every contranominal scale exactly once.
\end{theorem}

\begin{proof}
  We first show that the algorithm iterates over every contranominal scale of a formal context $\K=(G,M,I)$ at least once.
  Let $\mathbb{S}=\K[H,N]$ be a contranominal scale of dimension $k$ not computed by the algorithm that does not contain a smaller contranominal scale that is not computed by the algorithm.
  Let $H=\{(g_1,...,g_k\}$, $N=\{m_1,\ldots,m_k\}$ and $(g_i,m_j)\in I$ for all $i\not = j$.
  Without loss of generality $m_1\leq m_2 \leq ... \leq m_k$ is the lexicographic order on $N$.
  Consider the contranominal scale $\mathbb{\tilde{\SC}}=\K[\{(g_1,...,g_{k-1}\},\{m_1,\ldots,m_{k-1}\}]$ of dimension $k-1$ that is computed by the algorithm.
  Thereby, the generator of the characterizing tuples is given by $\{m_1,\ldots,m_{k-1}\}$.
  Thus, in the next iteration $m_k$ is added to this generator and $(g_k,m_k)$ is added to $C_N$.
  Due to the contranominal structure of $\mathbb{S}$ no element of $H$ is contained in the forbidden set $F$ and thus no element of $\{(g_1,m_1),(g_2,m_2),...,(g_{k-1},m_{k-1})\}$ is eliminated from $C_N$.
  Therefore $C_N$ corresponds to the characterizing tuples in the next step of the algorithm and the contranominal scale $\SC$ is reported.

  We now show that the algorithm iterates over every contranominal scale at most once.
  As the algorithm iterates over the generator attribute sets in lexicographical order, no attribute combination is iterated twice and every contranominal scale is reported at most once. \null\hfill$\square$
\end{proof}

\texttt{ContraFinder}, combined with \cref{mimcns}, can also be used to compute all maximum induced matchings in bipartite graphs.

\subsection{Speedup techniques}

\subsubsection{Clarifying and Reducing}
\label{sec:reduc-form-cont}

In the following, we consider clarified and reduced formal contexts with regards to reconstructing the contranominal scales in the original context from the contranominal scales of the augmented one.
This allows to use clarifying and reducing as a speedup technique.

In the clarified context, each pair of objects or attributes is merged if equality of their derivations holds.
To deduce the original formal context from the clarified one the previously merged attributes and objects can be duplicated.
Thus, contranominal scales containing merged objects or attributes are duplicated.

Now, we demonstrate how to reconstruct the contranominal scales from attribute reduced contexts.
Thereby, for each eliminated attribute $m$ we have to memorize the irreducible attribute set that has the same derivation as $m$.

\begin{definition}
    Let $\K=(G,M,I)$ be a formal context and $R(\K)$ the set of all attributes that are reducible in $\K$.
  Define the map $\omega\colon R(\K) \to \mathcal{P}(M\setminus R(\K))$ with $x \mapsto (N\subset M\setminus (R(\K)\cup\{ x\}))$ such that $N'=x'$ and $N$ of greatest cardinality.
  For a fixed object set $H\subseteq G$, let $\omega_H\colon R(\K) \to \mathcal{P}(M\setminus R(\K))$ be the map with $x \mapsto \{y\mid y\in \omega(x), \forall h \in H: (h,x) \not\in I \Rightarrow (h,y)\not \in I\}$.
\end{definition}

Note, that the map $\omega$ is well defined as the uniqueness follows directly from the maximality of $N$.
The following lemma provides a way to reconstruct the contranominal scales in the original context from the ones in the reduced one.

\begin{lemma}
Let $\K=(G,M,I)$ be a formal context with $\K_r$ its attribute-reduced subcontext and $\mathcal{K}$ the set containing all contranominal scales of $\K_r$.
Then the set
$\tilde{\mathcal{K}}=\{\K[H,\tilde{N}] \mid \K[H,N=\{n_1,\ldots, n_l\}] \in \mathcal{K}, \tilde{N}=\{\tilde{n}_i \mid  n_i = \tilde{n}_i \vee n_i \in \omega_H(\tilde{n}_i)\}\}$
contains exactly all contranominal scales of $\K$.
\end{lemma}

\begin{proof}
    Let $\SC=\K[H,\tilde{N}]$ be a contranominal scale in $\K$.
    Then each attribute $\tilde{n}_i \not\in R(\K)$ and thus an attribute of $\K_r$ or there is a unique minimal irreducible attribute set $U\subseteq M\setminus R(\K)$ with $\tilde{n}_i'=U'$ due to the definition of reducibility.
     In particular, for all $g\in G$ with $(g,\tilde{n}_i)\in I$ holds $(g,u)\in I$ for all $u\in U$.
     Furthermore, for all $g\in G$ with $(g,\tilde{n}_i)\not\in I$ there is at least one $u\in U$ with $(g,u)\not \in I$.
     Due to the contranominal property of $\SC$, for every $\tilde{n}_i\in N$ there is exactly one $h\in H$ with $(h,\tilde{n}_i)\not\in I$.
     Therefore there is at least one $n_i \in U$ with the same property and thus $\SC \in \tilde{\mathcal{K}}$.
     Now let $\SC=\K[H,\tilde{N}]$ be an element of $\tilde{\mathcal{K}}$.
     For each pair $(g,m)\not \in I$ there is no attribute $n$ such that $(g,n)\not \in I$ and there is no object $h$, such that $(h,n)\not \in I$.
     Due to the existence of an $g_i$ for each $m_i$ such that $(g_i,m_i) \not \in I$, the context $\SC$ is a contranominal scale.
     \null\hfill $\square$
\end{proof}

Thus, to reconstruct contranominal scales, for each $x\in R(\K)$ all $y\in\omega(x)$ are considered.
$U\cup x$ is a candidate for the attribute set of a contranominal scale in $\K$, if there is a $U\subset M\setminus \omega(x)$ with $U\cup y$ attribute set of a contranominal scale $\SC_y$ for all $y$.
This candidate forms the contranominal scale $\K[H,U\cup x]$, if and only if all contranominal scales $\SC_y$ share the same object set $H$.
The object reducible case can be done dually.

\subsubsection{Knowledge-Cores}
\label{sec:computing-k-k}

The notion of  $(p,q)$-cores is introduced to FCA by Hanika and Hirth in \cite{Hanika.2020}.
Thereby, dense subcontexts are defined as follows:

\begin{definition}[Hanika and Hirth \cite{Hanika.2020}]
  Let $\K= (G, M, I)$ and $\mathbb{S}= \K[H, N]$ be formal contexts.
  $\mathbb{S}$ is called a \emph{$(p,q)$-core} of $\K$ for $p, q \in \N$, if
  $\forall g \in H: |g'|\geq p$ and $\forall m \in N: |m'|\geq q$ and
  $\SC$ is maximal under this condition.
\end{definition}

Every formal context with fixed $p$ and $q$ has a unique $(p,q)$-core.
Computing knowledge cores provides a way to reduce the number of attributes and objects in a formal context without removing large contranominal scales.

\begin{lemma}
Let $\K$ be a formal context, $k\in \N$, and $\mathbb{S}\le\K$ its $(k-1,k-1)$-core. Then for every contranominal scale $\mathbb{C}\le\K$ of dimension $k$ it holds $\mathbb{C}\leq \mathbb{S}$.
\end{lemma}

\begin{proof}
  Assume not; i.e., there is a contranominal $\K[H,N] \not \subset \mathbb{S}=\K[H_S,N_S]$.
  But then $\K[H_S\cup H, H_S\cup N]$ is a $(k-1,k-1)$-core of $\K$ and $\mathbb{S} \leq \K[H_S\cup H, H_S\cup N]$, contradicting the definition of $(k-1,k-1)$-cores. \null\hfill$\square$
\end{proof}

Thus, to compute all contranominal scales of dimension at least $k$ it is possible to compute them in the $(k-1,k-1)$-core.
Note that in this case however, smaller contranominal scales might get eliminated.
Therefore, if the goal is to compute contranominal scales of smaller sizes the $(k-1,k-1)$-cores should not be computed.

\section{Attribute Selection}
In this section we propose \texttt{$\delta$-adjusting}, a method to select attributes based on measuring their influence for contranominal scales as follows:

\begin{definition}
 Let $\K=(G,M,I)$ be a formal context and $k\in\N$.
 Call $N \subset M$ \emph{$k$-cubic} if $\exists H \subset G$ with $\K[H,N]$ being a contranominal scale of dimension $k$ and $\nexists \tilde{N} \supseteq N$ such that $\tilde{N}$ is $(k+1)$-cubic.
  Define the \emph{contranominal-influence}
 of $m\in M$ in $\K$ as $\zeta(m)\coloneqq \sum_{k=1}^\infty\left(|\{N \subset M \mid m \in N, N \text{is $k$-cubic}\}|\cdot \frac{2^k}{k}\right).$
\end{definition}

Subcontexts that are $k$-cubic are directly influencing the concept lattice, as those dominates the structure as the following shows.

\begin{proposition}
  An attribute set is $k$-cubic, iff the sub-meet-semilattice that is generated by its attribute concepts is  a Boolean lattice of dimension $k$ that has no Boolean superlattice in the original concept lattice.
\end{proposition}

The contranominal influence thus measures the impact of an attribute on the lattice structure.
In this, only the maximal contranominal scales are considered since the smaller non maximal-ones have no additional structural impact.
As each contranominal scale of dimension $k$ corresponds to $2^k$ concepts, we scale the number of attribute combinations with this factor.
To distribute the impact of a contranominal scale evenly over all involved attributes, the measure is scaled by $\frac{1}{k}$.
With this measure we now define the notions of \texttt{$\delta$-adjusting}.

\begin{definition}
  Let $\K=(G,M,I)$ be a formal context and $\delta\in[0,1]$.
  Let $N \subset M$ minimal such that $\frac{|N|}{|M|}\geq\delta$, $\zeta(n)<\zeta(m)$ for all $n\in N, m\in M\setminus N$.
  We call $\mathbb{A}_{\delta}(\K)\coloneqq \K[G,N]$ the \emph{$\delta$-adjusted subcontext} of $\K$ and $\underline{\BB}(\A_{\delta}(\K))$ the \emph{$\delta$-adjusted sublattice} of $\underline{\BB}(\K)$.
\end{definition}

Note, that \texttt{$\delta$-adjusting} always results in unique contexts.
Moreover, every \texttt{$\delta$-adjusted} sublattice is a sub-meet-semilattice of the original one \cite[Prop 31]{fca-book}. For every context $\K=(G,M,I)$ it holds that $\A_1=\K$ and
$\A_0=\K[G,\emptyset]$.
A context from a medical diagnosis dataset with measured contranominal influence and computed \texttt{$\frac{1}{2}$-adjusted} subcontext can be retraced in \cref{runexp_Kontext}.

\begin{figure}[!ht]
\null\hfill
                \begin{cxt}%
                        \att{a}%
                        \att{b}%
                        \att{c}%
                        \attc{d}%
                        \attc{e}%
                        \att{f}%
                        \att{g}%
                        \attc{h}%
                        \attc{i}%
                        \attc{j}%
                        \att{k}%
                        \attc{l}%
                        \att{m}%
                        \attc{n}%
                        \attc{o}%
                        \obj{..xcc.xcPc.P.Pc}{111} %
                        \obj{x..PcxxcPP.cxPc}{119} %
                        \obj{.xxcc..PccxPxPP}{31} %
                        \obj{x.xPc..ccc.P.Pc}{32} %
                        \obj{.x.ccx.PccxPxPP}{17} %
                        \obj{.XXcc..ccc.PxPP}{27} %
                        \obj{XX.PPXXPPPXcXPP}{105} %
                        \obj{.X.ccX.PPcXPXcP}{58} %
                        \obj{X..PcX.cPP.cXcc}{65} %
                        \obj{X.XPP.XPPPXc.Pc}{103} %
                        \obj{.XXcc..PPcXPXcP}{56} %
                        \obj{XXXPP.XPPPXcXPP}{98} %
                        \obj{.XXcc..cPc.PXcP}{43} %
                        \obj{X.XPc..cPc.P.cc}{50} %
                \end{cxt}
\hfill
    \begin{tabular}{llcccc}
    \toprule
     &Attribute Name &2&3&4& $\zeta$ \\
    \midrule
        a:& Lumbar pain y & 1 & 22 & 6 &  84.7 \\
        b:& Bladder inflammation y& 1 & 29 & 0 & 79.3\\
        c:& Burning n & 1 & 31 & 9 & 120.7 \\
        \cb d:&\cb Lumbar pain n &\cb 2 &\cb 19 &\cb 0 &\cb 54.7\\
        \cb e:&\cb Nausea n &\cb 0 &\cb 16 &\cb 3 &\cb 54.7\\
        f:& Burning y & 1 & 31 & 0 & 84.7\\
        g:& Temp. $\in [40.0,~ 42.0]$& 2 & 24 & 5 & 88.0\\
        \cb h:&\cb Micturition pains n &\cb 1 &\cb 18 &\cb 5 &\cb 70.0\\
        \cb i:&\cb Temp. $\in[35.0,~ 37.5]$ &\cb 3 &\cb 16 &\cb 0 &\cb 48.7\\
        \cb j:&\cb Pelvis nephritis n &\cb 1 &\cb 19 &\cb 1 &\cb 56.7\\
        k:& Micturition pains y & 1 & 33 & 0 & 90.0\\
        \cb l:&\cb Pelvis nephritis y &\cb 3 &\cb 17 &\cb 0 &\cb 51.3\\
        m:& Urine pushing y & 0 & 21 & 7 & 84.0\\
        \cb n:&\cb Temp. $\in[37.5,~ 40.0]$ &\cb 2 &\cb 23 &\cb 3 &\cb 77.3\\
        \cb o:&\cb Bladder inflammation n &\cb 1 &\cb 26 &\cb 1 &\cb 75.3\\
        \bottomrule
    \end{tabular}
\hfill\null\\
\null\hfill
        \begin{minipage}{0.42\textwidth}
                \includegraphics[width=\textwidth]{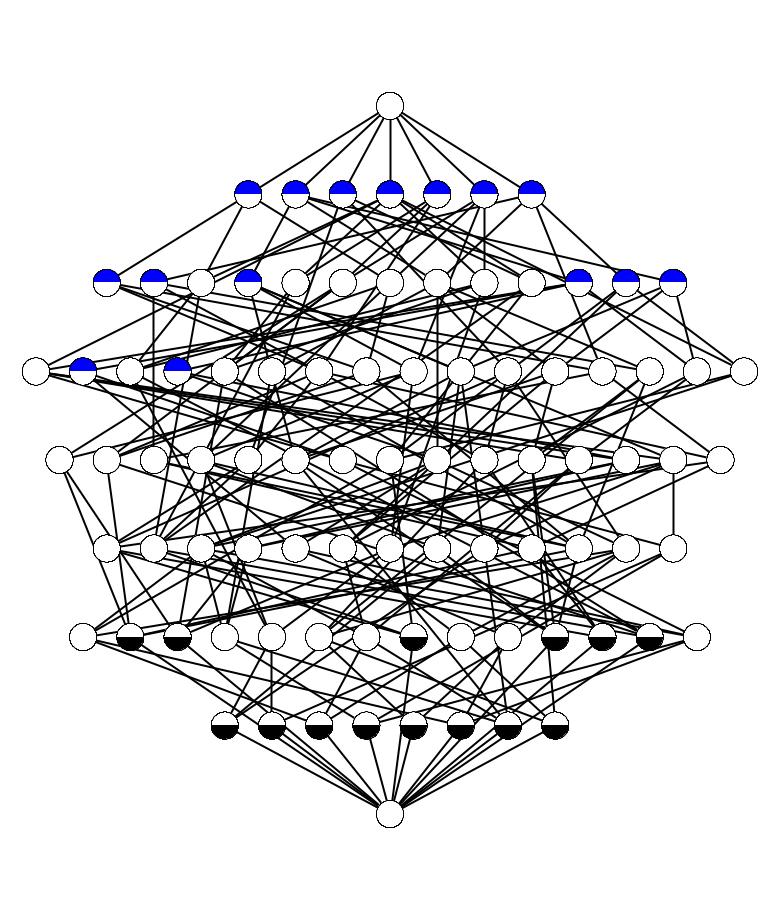}
        \end{minipage}
        \hfill
        \begin{minipage}{0.32\textwidth}
                \includegraphics[width=\textwidth]{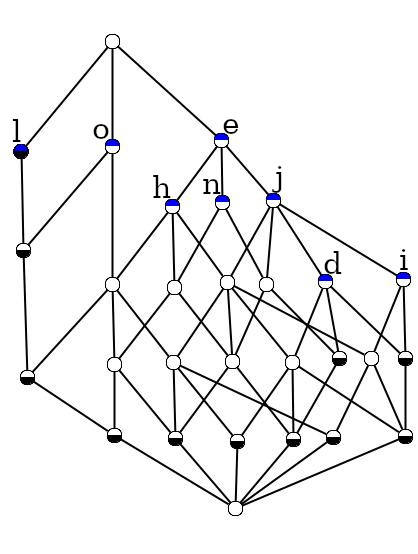}
        \end{minipage}
        \hfill\null
    \caption{Top: Reduced and clarified medical diagnosis dataset~\cite{diagnosis}. The \texttt{$\frac{1}{2}$-adjusted} subcontext is highlighted. The objects are patient numbers. The attributes are described in the figure together with the count of k-cubic subcontexts and their contranominal influence~$\zeta$.
    Bottom: Lattice of the original (left) and the \texttt{$\frac{1}{2}$-adjusted} (right) dataset.}
    \label{runexp_Kontext}
\end{figure}

It is important to observe that for a context $\K$ and its reduced context $\K_r$ a different attribute set can remain if they are \texttt{$\delta$-adjusted}, as can be seen in \cref{fig:reducedCN}.
Therefore, the resulting concept lattices for $\K$ and $\K_r$ can differ.
To preserve structural integrity between \texttt{$\delta$-adjusted} formal contexts and their concept lattices we thus recommend to only consider clarified and reduced formal contexts.
In the rest of this work, these steps are therefore performed prior to \texttt{$\delta$-adjusting.}
Note, that since no attributes are generated no new contranominal scales can arise by \texttt{$\delta$-adjusting}.
Furthermore, removing attributes can not turn another attribute from irreducible to reducible.
On the other hand however, objects can become reducible as can be seen again in \cref{fig:reducedCN}.
While $6$ is irreducible in the original context, it is reducible in $\A_{\frac{3}{5}}(\K)$.

\begin{figure}[t]
        \begin{minipage}{0.26\textwidth}
        \centering
            \includegraphics[height=10.8em]{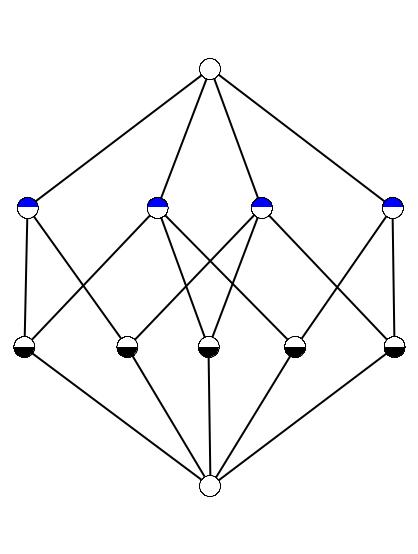}
        \end{minipage}
    \begin{minipage}{0.23\textwidth}
                \centering
                \begin{cxt}%
                        \att{a}%
                        \att{b}%
                        \att{c}%
                        \attc{d}%
                        \attc{e}%
                        \obj{xx.PP}{1} %
                        \obj{x.xPP}{2} %
                        \obj{.xxPP}{3} %
                        \obj{..xcc}{4} %
                        \obj{...cP}{5} %
                        \obj{.x.cP}{6} %
                \end{cxt}
        \end{minipage}
        \begin{minipage}{0.10\textwidth}
        \centering
            \includegraphics[width=\textwidth]{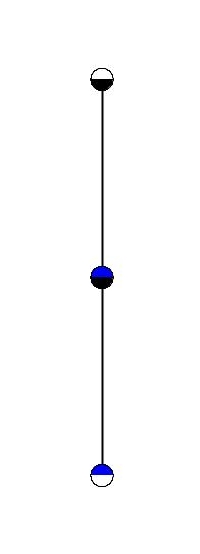}
        \end{minipage}
        \begin{minipage}{0.25\textwidth}
        \centering
                \begin{cxt}%
                        \att{a}%
                        \att{b}%
                        \att{c}%
                        \attc{d}%
                        \obj{xx.P}{1} %
                        \obj{x.xP}{2} %
                        \obj{.xxP}{3} %
                        \obj{..xc}{4} %
                        \obj{...c}{5} %
                        \obj{.x.c}{6} %
                \end{cxt}
        \end{minipage}
        \begin{minipage}{0.1\textwidth}
        \centering
            \includegraphics[width=\textwidth]{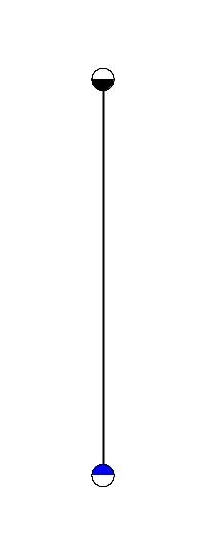}
        \end{minipage}
        \caption{A concept lattice together with two of its contexts $\K$ and $\K_r$ whereby $\K_r$ is attribute reduced while $\K$ contains the reducible element $e$.
        In both contexts the \texttt{$\frac{3}{5}$-adjusted} subcontext is highlighted.
        Their lattices (right to each context) differ.}
        \label{fig:reducedCN}
\end{figure}

\subsection{Properties of Implications}

In this section we investigate \texttt{$\delta$-adjusting} with respect to the influence on implications.
Let $\K=(G,M,I)$ be a formal context, $m \in M$ and $X\rightarrow Y$ an implication in $\K$.
If $m$ is part of the implication; i.e., $m \in X$ or $m \in Y$, this implication vanishes.
Therefore the removal of $m$ in an implication $X\rightarrow Y$ of some implication base $\mathcal{C}(\K)$ is of interest.
If $m$ is neither part of a premise nor a conclusion of an implication $X\rightarrow Y\in \mathcal{C}(\K)$ its removal has no impact on this implication base.
In case $m\in Y$, its elimination changes all implications $X \rightarrow Y$ to $X\rightarrow Y\setminus \{m\}$.
Note that, even though all implications can still be deduced from $\mathcal{C}' = \{X \rightarrow Y: X \rightarrow Y \cup\{m\} \in \mathcal{C}(\K)\}$ this set is not necessarily minimal and in this case is not a base.
Especially if $\{m\}=Y$ the resulting $X\rightarrow\emptyset$ is never part of an implication base.
In case $m \in X$, every $Z\rightarrow X$ in the base is changed to $Z \rightarrow X\setminus \{m\} \cup Y$ while $X\rightarrow Y$ is removed.
Similarly to the conclusion case, the resulting set of implications can be used to deduce all implications but is not necessarily an implication base.
Moreover, no new implications can emerge from the removal of attributes, as the following shows.

\begin{lemma}
        Let $\K=(G,M,I)$ be a formal context, $N\subset M$ and $X,Y\subseteq N$ with $X\rightarrow Y$ a non-valid implication in $\K$.
        Then $X\rightarrow Y$ is also non-valid in $\K[G,N]$.
\end{lemma}

\begin{proof}
        Since $X\rightarrow Y$ is not valid in $\K$, there exists an object $g\in G$ with $X\subseteq g'$ and $Y\not\subseteq g'$.
        As the objects in $\K$ and $\K[G,N]$ are identical on $N$ (especially if $X,Y\subseteq N$), g is a counterexample for $X\rightarrow Y$ in $\K[G,N]$. \hfill$\square$
\end{proof}

Thus, the relationship between the implications of a subcontext with all objects and the original context is as follows:

\begin{corollary}
        Let $\K=(G,M,I)$ be a formal context, $\SC=\K[G,N]$ and $N \subset M$.
        Then $Imp(\SC)\subseteq Imp(\mathbb{K})$.
\end{corollary}

This influences the size of the base of a \texttt{$\delta$-adjusted} subcontext as follows:

\begin{lemma}
    Let $\K=(G,M,I)$ a formal context, and $\mathbb{S}=\K[G,N]$ and $N\subset M$.
    Then $|\mathcal{C}(\mathbb{S})| \le |\mathcal{C}(\mathbb{K})|$.
\end{lemma}

\begin{proof}
Assume not; i.e., the $|\mathcal{C}(\mathbb{S})| > |\mathcal{C}(\mathbb{K})|$. Let $J$ be the set of implications containing $m\in M\setminus N$ in $\mathcal{C}(\mathbb{K})$. A set of implications that can generate the whole implication set with size $|\mathcal{C}(\mathbb{K})|$ or less is given by altering the implications $X\rightarrow Y$ $X,Y\in M$ in $|\mathcal{C}(\mathbb{K})|$ as follows. If $m\in Y$, $X\rightarrow Y$ is replaced by $X\rightarrow Y\setminus m$. If $m\in X$ and $Z\rightarrow X \in \mathcal{C}(\K)$, $X\rightarrow Y$ is replaced by $Z\rightarrow Y$. This yields a contradiction.  \hfill$\square$
\end{proof}

Revisiting the context in \cref{runexp_Kontext} together with its \texttt{$\frac{1}{2}$-adjusted} subcontext the selection of nearly $50\%$ of the attributes (8 out of 15) results in a sub-meet-semilattice containing only $33\%$ of the concepts (29 out of 88).
Moreover, the implication base of the original context includes 40 implications.
After the alteration its size is decreased to 11 implications.

\section{Evaluation and Discussion}
\label{sec:evaluation}

In this section we evaluate the algorithm \texttt{ContraFinder} and the process of \texttt{$\delta$-adjusting} using real-world datasets.

\subsection{Datasets}
\begin{table}[b]
    \centering
    \caption{Datasets used for the evaluation of \texttt{ContraFinder} and \texttt{$\delta$-adjusting}.}
    \label{tab:descriptiv}
    \begin{tabular}{lrrrrr}
    \toprule
        & Zoo & Students &  Wikipedia & Wiki44k & Mushroom \\
        \midrule
        Objects: & 101 & 1000 & 11273 & 45021 & 8124\\
        Attributes: & 43 & 32 & 102 & 101& 119\\
        Density: & 0.40 & 0.28 & 0.015 & 0.045& 0.19\\
        Number of concepts: & 4579  & 17603 & 14171 & 21923& 238710\\
        Mean objects per concept: & 18.48 & 16.73 & 20.06 & 109.47 & 91.89\\
        Mean attributes per concept: & 7.32 & 5.97 & 5.88 & 7.013 & 16.69 \\
        Size of canonical base: & 401 & 2826 & 4575 & 7040 & 2323 \\
         \bottomrule
    \end{tabular}
  \end{table}
\cref{tab:descriptiv} provides descriptive properties of the datasets used in this work.
The zoo~\cite{Dua:2019,rowley1990pc} and mushroom~\cite{Dua:2019,schlimmer1981mushroom} datasets are classical examples often used in FCA based research such as the \texttt{TITANIC} algorithm.
The Wikipedia~\cite{kunegis2013konect} dataset depicts the edit relation between authors and articles while the
Wiki44k dataset is a dense part of the Wikidata knowledge graph.
The original wiki44k dataset was taken from \cite{Ho.2018}, in this work we conduct our experiments on an adapted version by \cite{Hanika.2019}.
Finally, the Students dataset~\cite{students} depicts grades of students together with properties such as parental level of education.
All experiments are conducted on the reduced and clarified versions of the contexts.
For reproducibility the adjusted versions of all datasets are published in \cite{dataset}.

\subsection{Runtime of ContraFinder}

\texttt{ContraFinder} is a recursive backtracking algorithm that iterates over all attribute sets containing contranominal scales.
Thus, the worst case runtime is given by $O(n^k)$ where $n$ is the number of attributes of the formal context and $k$ the maximum dimension of a contranominal scale in it.
The Branch-And-Search algorithm from \cite{Xiao2017} has a runtime of $O(1.3752^n)$ where $n$ is the sum of attributes and objects.
Finally the Bron-Kerbosch algorithm has a worst-case runtime of $O(3^{n/3})$ with $n$ being the number of non-incident object-attribute pairs.

To compare the practical runtime of the algorithms we test them on the previously introduced real world datasets.
We report the runtimes in \cref{tab:runtime}, together with the dimension of the larges contranominal scale and the total number of contranominal scales.
Note, that for larger datasets we are not able to compute the number of all contranominal scales using Bron-Kerbosch (from Students) and the Branch-And-Search algorithm (Mushroom) below 24 hours due to their exponential nature and thus stopped the computations.
All experiments are conducted on an Intel Core i5-8250U processor with 16 GB of RAM.

\begin{table}[!b]
    \centering
    \caption{Experimental runtimes of the different algorithms on all datasets.}
    \label{tab:runtime}
    \begin{tabular}{lrrrrr}
    \toprule
        & Zoo & Students &  Wikipedia & Wiki44k & Mushroom \\
        \midrule
        ContraFinder: & 2.43 & 7.36 & 17.15 & 35.65 & 1961.0 \\
        Bron Kerbosch searching cliques:  & 138.70 & >86400 & >86400 & >86400 & >86400\\
        Branch and Search algorithm: & 14.40 & 12005.82 & 1532.17 & 16783.58 & >86400 \\
        \midrule
        Dim. of max.\ contranominal scale: & 7 & 8 & 9 & 11 & 10  \\
        Number of contranominal scales: &$4.1\cdot10^7$ & $7.8\cdot 10^9$ & $9.9\cdot 10^8$ & $2.0\cdot 10^{14}$ & $1.2\cdot 10^{19}$

 \\
        \bottomrule
    \end{tabular}
\end{table}

\begin{table}[t]
    \centering
    \caption{Evaluation of $k$-adjusted contexts. The standard deviation is given in parenthesis. "Acc of DT" is the abbreviation for "Accuracy of the Decision Tree".}
    \label{tab:experiments}
    \begin{tabular}{llrrrrr}
    \toprule
        && Zoo & Students &  Wikipedia & Wiki44k & Mushroom \\
        \midrule
        $|\BB(\K)|$: &$\frac{1}{2}$-adjusted:& \textbf{90} & \textbf{312} & \textbf{65} & 323 & \textbf{426} \\
        &Sampling: & 496 (205) & 1036 (327) & 833 (517) & 1397 (627) & 8563 (4532)  \\
        &Hanika et.al:& 95 & 341 & 67 & \textbf{254} & 561  \\
        \midrule
        $|\mathcal{C}(\K)|$: &$\frac{1}{2}$-adjusted:& 98 & \textbf{105} & 626 & \textbf{1003} & \textbf{339} \\
        &Sampling:& \textbf{95} (17) & 156 (35) & 758 (101)& 1360 (135) & 574 (93) \\
        &Hanika et.al:& 100 & \textbf{105} & \textbf{553} & 1091 & 490  \\
        \midrule
        Acc of DT: &$\frac{1}{2}$-adjusted:& 0.88 (0.08) & 0.88 (0.06) & \textbf{0.99} (0.01) & \textbf{0.98} (0.03) & \textbf{0.98} (0.02) \\
        &Sampling:& \textbf{0.89} (0.15) &0.81 (0.15) & 0.9 (0.14) & 0.95 (0.06) &
        0.92 (0.13) \\
        &Hanika et.al: & 0.88 (0.09) & \textbf{0.89} (0.06) & \textbf{0.99} (0.01)
        & \textbf{0.98} (0.16) & 0.97 (0.03) \\
        \bottomrule
    \end{tabular}
\end{table}

\subsection{Structural Effects of $\delta$-Adjusting}
We measure the number of formal concepts generated by the formal context as well as the size of the canonical base.
To demonstrate the effects of $\delta$-adjusting we focus on $\delta=\frac{1}{2}$.
Our two baselines are selecting the same number of attributes using random sampling and choosing the attributes of highest relative relevance as described in \cite{hanika2019relevant}.
It can be observed, that in all three cases the number of concepts heavily decrease.
However, this effect is considerably stronger for \texttt{$\frac{1}{2}$-adjusting} and the approach of Hanika et.al.\ compared to sampling.
Hereby, \texttt{$\frac{1}{2}$-adjusting} yields smaller concept lattices on four datasets.
A similar effect can be observed for the sizes of the canonical bases where our method yields three times in the smallest cardinality.

\subsection{Knowledge in the $\delta$-Adjusted Context}
To measure the degree of encapsulated knowledge in \texttt{$\delta$-adjusted} formal contexts we conduct the following experiment using once again sampling and the relative relevant attributes of Hanika et.al.\  as baselines.
In order to measure if the remaining subcontexts still encapsulates knowledge we train a decision tree classifier on them predicting an attribute that is removed beforehand.
This attribute is sampled randomly in each step.
To prevent a random outlier from distorting the result we repeat this same experiment 1000 times for each context and method and report the mean value as well as the standard-deviation in \cref{tab:experiments}.
The experiment is conducted using a 0.5-split on the train and test data.
For all five datasets, the results of the decision tree on the \texttt{$\frac{1}{2}$-adjusted} context are consistently high, however \texttt{$\frac{1}{2}$-adjusting} and the Hanika et.al.\ approach outperform the sampling approach.
Both this methods achieve the highest score on four contexts, in two of this cases the highest result is shared.
The single highest score of sampling is just slightly above the other two approaches.

\subsection{Discussion}

The theoretical runtime of \texttt{ContraFinder} is polynomial in the dimension of the maximum contranominal.
Therefore, compared to the baseline algorithms it performs better, the smaller the maximum contranominal scale in a dataset.
Furthermore, the runtime of Bron-Kerbosch is worse, the sparser a formal context, as the number of pairs that are non-incident increases and thus more vertices have to be iterated.
Finally, the Branch-And-Search algorithm is best in the case that the dimension of the maximum contranominal scale is not bounded.
To evaluate, how this theoretical properties translate to real world data, we compute the set of all contranominal scales with the three algorithms on the previously described datasets.
Only \texttt{ContraFinder} can compute the set of all contranominal scales on the larger datasets on our hardware under 24 hours.
The runtime of \texttt{ContraFinder} is thus superior to the other two on real-world datasets.

To evaluate the impact on the understandability of the \texttt{$\delta$-adjusted} formal contexts, we conduct the experiments measuring the sizes of the concept lattices and the canonical bases.
All three evaluated methods heavily decrease the size of the concept lattice as well as the canonical base.
Compared to the random sampling \texttt{$\frac{1}{2}$-adjusting} and the method of Hanika et.al.\ influence the size of this structural components much stronger.
Among those two, \texttt{$\frac{1}{2}$-adjusting} seems to slightly outperform the method of Hanika et.al.\ and is thus more suited to select attributes from a large dataset in order to be analyzed by a human.

To evaluate to what extent knowledge in the formal context of reduced size is encapsulated we conduct the experiment with the decision trees.
This experiment demonstrates that the selected formal subcontext can be used in order to deduce relationships of the remaining attributes in the context.
While meaningful implications are preserved and the implication set is downsized, \texttt{$\frac{1}{2}$-adjusted} lattices seem to be suitable to preserve large amounts of data from the original dataset.
Similar good results can be achieved with the method of Hanika et.al.; however, our algorithm combines this with producing smaller concept lattices and canonical bases and is thus more suitable for the task to prepare data for a human analyst by reducing sizes of structural constructs.

We conclude from these experiments that \texttt{$\delta$-adjusting} is a solution to the problem to make information more feasible for manual analysis while retaining important parts of the data.
In particular, if large formal contexts are investigated this method provides a way to extract relevant subcontexts.

\section{Conclusion}
\label{sec:conclusion}
In this work, we proposed the algorithm \texttt{ContraFinder} in order to enable the computation of the set of all contranominal scales in a formal context.
Using this, we defined the contranominal-influence of an attribute.
This measure allows us to select a subset of attributes in order to reduce a formal context to its\texttt{$\delta$-adjusted} subcontext.
The size of its lattice is significantly reduced compared to the original lattice and thus enables researchers to analyze and understand much larger datasets using Formal Concept Analysis.
Furthermore, the size of the canonical base, which can be used in order to derive relationships of the remaining attributes shrinks significantly.
Still, remaining data can be used to deduce relationships between attributes, as our classification experiment shows.
This approach therefore identifies subcontexts whose sub-meet-semilattice is a restriction of the original lattice of a formal context to a small meaningful part.

Further work in this area could leverage \texttt{ContraFinder} in order to compute the contranominal-relevance of attributes more efficiently to handle even larger datasets.
Moreover, a similar measure for objects could be introduced.
However, one should keep in mind that hereby false implications can arise.

\bibliographystyle{splncs04}
\bibliography{paper}
\end{document}